\documentclass[11pt]{article}
\pdfoutput=1
\usepackage[margin=1in]{geometry}
\usepackage{amsmath,amssymb,amsthm}
\usepackage{booktabs}
\usepackage{multirow}
\usepackage{graphicx}
\usepackage[numbers]{natbib}
\usepackage[colorlinks=true,linkcolor=blue,citecolor=blue,urlcolor=blue]{hyperref}
\usepackage{xcolor}
\usepackage{placeins}
\usepackage{pifont}
\usepackage{tikz}
\usetikzlibrary{positioning,arrows.meta}
\hypersetup{
  pdftitle={Anatomy of a Sound Neural Reasoner: One-Shot Amortization,
            First-Pass Poisoning, and Search Inertness in Clue-Rich Completion},
  pdfauthor={Aleksey Komissarov}
}
\newcommand{\colt}{\textsc{CoLT}}
\newcommand{\meet}{\sqcap}
\newcommand{\bottom}{\bot}

\newtheorem{observation}{Observation}
\newtheorem{proposition}{Proposition}

\title{Anatomy of a Sound Neural Reasoner:\\ One-Shot Amortization, First-Pass Poisoning,\\ and Search Inertness in Clue-Rich Completion}

\author{
  Aleksey Komissarov\thanks{Corresponding author: \texttt{a.komissarov@nup.ac.cy}}\\
  Neapolis University, Pafos, Cyprus\\
  \\
  \small Code, frozen protocol, and all raw results:
  \url{https://github.com/ad3002/colt/tree/r14-2026-07-15}
}
\date{July 2026}

\begin{document}
\maketitle

\begin{abstract}
Neural solvers are built to deduce, branch, and revise intermediate states. The Lattice
Deduction Transformer (LDT) appears to do exactly that. In clue-rich Sudoku, it does not:
one forward pass commits essentially the entire grid (every blank cell on standard
6$\times$6, 94--96\% on augmented 9$\times$9), turning the iterative solver into a
one-shot
predictor wrapped in an exact verifier. All hard-slice failures are decided before
search begins, when the first pass confidently deletes a value required by the true
solution. We call this \emph{first-pass poisoning}. Adding learned branching, MRV,
backtracking, value exclusion, and shared nogoods (\colt{}) does not change which Sudoku
instances are solved; it cuts repeated invalid derivations \textbf{1{,}497-fold}. At the frozen training budget, constraint-graph attention alone matches full-\colt{}
accuracy, while positional tables recover only under substantially longer training: an
optimization and sample-efficiency advantage, not an absolute capacity difference. The diagnosis predicts two
cures, and both work. Digit-permutation augmentation raises 9$\times$9 accuracy from below 1\% to
\textbf{96.5\% $\pm$ 0.3} across three training seeds on a symmetry-disjoint split. Test-time union over symmetry-transformed
passes raises all three hard-slice checkpoints from 72.8--78.9\% to \textbf{100\%}
without retraining. On
from-scratch graph coloring, one-shot behavior disappears and search changes accuracy. In
clue-rich completion, LDT-like systems are one-shot amortized predictors rather than
learned search procedures: accuracy is set by calibration and symmetry; search mainly
removes computational waste.
\end{abstract}

\section{Introduction}

Two May-2026 architectures attack small-model reasoning from opposite ends. GRAM
\citep{baek2026gram} takes the \emph{proposal} view: a deterministic recursive reasoner in the
HRM/TRM lineage \citep{wang2025hrm,jolicoeur2025trm} is wrapped in a variationally trained
stochastic guidance term, so the latent trajectory explores; auxiliary heads rank and halt
trajectories. Nothing, however, prevents a confidently wrong answer. The Lattice Deduction
Transformer \citep{davis2026ldt} takes the \emph{soundness} view: the state is an explicit
lattice of per-cell candidate sets; the network only ever \emph{removes} candidates; conflicts
($\bottom$) are detected, branches are committed, and a completed grid is emitted only after an
external constraint check; otherwise the model abstains. Throughout this paper
``soundness'' means this \emph{verified-emission} property: no emitted answer fails
the exact external check. It is narrower than soundness in the formal-methods sense: the
neural conflict signal is heuristic evidence rather than proof, value exclusion can prune
true paths, completeness is not guaranteed, and the system may abstain; what is guaranteed
is the absence of false emitted solutions. LDT reports 100\%/100\%
accuracy/soundness on Sudoku-Extreme, where frontier LLMs score 0\%.

We reimplemented both systems from their papers (ad3002/gram, ad3002/LTD;
\citealp{komissarov2026gram_reimpl,komissarov2026ltd_reimpl}) and measured their failure modes
on shared data. What decided the direction of this paper is where LDT's search compute goes.
On our 4$\times$4 reproduction, the inference solver suppressed \textbf{35{,}200}
constraint-violating completions, concentrated on the 11/300 puzzles it failed: the same
wrong grids re-derived hundreds of times each. At 6$\times$6 the count was \textbf{148{,}567}.
The soundness guard holds; the search spends almost all of its budget re-deriving known
dead ends. Structurally this is no surprise: branch cells are picked uniformly at random,
a conflict discards the whole trajectory (restart from the initial state), and the $K$
parallel chains share no memory.

\begin{observation}[Soundness licenses the search layer]
\label{obs:licensing}
In a verify-or-abstain system, soundness is a property of the emission rule alone: an answer
is emitted only after passing an exact external check. Consequently every search policy,
whether learned, heuristic, or adversarial, preserves the guarantee, and the search layer is
a free design space.
\end{observation}

Classical SAT/CSP solving tells us what to put in that free space: conflict-driven
clause learning, backjumping, and informed variable ordering turned satisfiability solving
from hopeless to routine \citep{silva1999grasp,moskewicz2001chaff}. \colt{} instantiates a
CDCL-\emph{inspired} subset of that playbook neurally, inside the sound envelope
(Table~\ref{tab:correspondence}): informed ordering, chronological backtracking, and
verified leaf-level nogoods, but not conflict analysis, which a black-box propagator cannot
support (Limitations~iii). The
hypothesis under test is that this search layer is where the missing accuracy lives. In the
clue-rich Sudoku regimes studied here, the rest of the paper refutes that hypothesis:
accuracy is set by the calibration of the propagator's first forward pass. The boundary
study (\S\ref{sec:boundary}) maps where this stops holding.

\begin{table}[h]
\centering\small
\begin{tabular}{ll}
\toprule
\textbf{Classical CDCL solver} & \textbf{\colt{}} \\
\midrule
unit propagation & recurrent transformer over the lattice (from LDT) \\
variable ordering (MRV/VSIDS) & \textbf{policy head} trained on branch-survival probability \\
value ordering & branch distribution $\mathrm{softmax}(b/\tau)$ (from LDT) \\
conflict $\Rightarrow$ backtrack & \textbf{decision-stack DFS with value exclusion} \\
learned clauses (nogoods) & \textbf{per-puzzle ban set} of verified-wrong completions \\
restarts & stack-exhaustion restart with sampled orderings \\
state heuristics & \textbf{value function} $V(x) = 1 - \sigma(\text{conflict head})$ \\
\bottomrule
\end{tabular}
\caption{The CDCL-inspired correspondence. Bold components are new in \colt{}; the rest is
inherited from LDT \citep{davis2026ldt}. GRAM's \citep{baek2026gram} contribution is the
perspective that search decisions are a trainable distribution, not a fixed heuristic. The
correspondence is architectural, not complete: \colt{} has no conflict \emph{analysis}, so
its backtracking is chronological and its learned nogoods are whole verified-wrong grids
rather than general clauses (Limitations~iii).}
\label{tab:correspondence}
\end{table}

\paragraph{Contributions.} We built \colt{}'s search layer deliberately strong so that
its measured effect on accuracy would settle the hypothesis in either direction.
\begin{enumerate}
  \item \textbf{A mechanistic anatomy of where sound lattice reasoning works and why it
  fails} (\S\ref{sec:anatomy}--\ref{sec:amortized}, \S\ref{sec:phase4}): a single forward
  pass commits all blank cells to singletons (rate 1.0), so in the clue-rich regimes
  tested these systems are one-shot
  amortized solvers inside a verifier shell. One-shot behavior is itself the optimum of
  unique-solution supervision (\S\ref{sec:amortized}); the contribution is measuring how
  completely it is attained, tying failures to poisoning, and confirming the predicted
  interventions; a perfect, threshold-insensitive contingency
  (43/43 vs.\ 0/137, replicated at 42/42 vs.\ 0/138 on different hardware) attributes
  all 43 failures to first-pass poisoning; and a symmetry
  argument (Proposition~\ref{prop:symmetry}) shows the exact depth-0 target contains no
  elimination signal. Two predicted cures both confirm: train-time digit-permutation augmentation
  (9$\times$9/25-clue, below 1\% to 96.5 $\pm$ 0.3\% over three seeds) and test-time
  union symmetry-frame ensembling (hard slice, 76.1\% to 100\% on each of three seeds;
  they compose to 180/180 at 9$\times$9), with a 2$\times$2 control isolating the
  interaction that makes augmentation
  affordable only under the \colt{} parameterization.
  \item \textbf{A bounded negative result} (\S\ref{sec:results}): a CDCL-inspired
  neural search layer, a branch-policy head with dense free supervision derived from LDT's
  own abstraction operator, decision-stack DFS with value exclusion, and a cross-chain ban
  set of verified leaf-level nogoods, each independently switchable to give a 2$\times$3
  ablation grid
  \{restart, dfs\}$\times$\{random, MRV \citep{haralick1980mrv}, learned\} from one
  checkpoint, has no detectable effect on accuracy in any of the clue-rich Sudoku settings
  tested (a bounded null: at
  $n{=}180$ the 95\% interval half-width reaches $\pm 7$\,pp near 50\%, and the anatomy
  explains the null mechanistically),
  while cutting wasted wrong-completion derivations by three orders of magnitude (74{,}864
  to 50 on the hard slice). In the clue-rich completion regime where amortization operates, search buys
  efficiency and propagator calibration bounds accuracy; the boundary study
  (\S\ref{sec:boundary}) marks where this inverts (from-scratch instances at low
  constrainedness, the one place DFS buys accuracy: +2.9\,pp pooled over three seeds,
  up to +11.4\,pp in single-seed sweeps).
  \item \textbf{The instruments} (\S\ref{sec:method}--\ref{sec:protocol},
  Appendix~\ref{app:audit}): \colt{} itself, whose constraint-graph attention bias (the
  Graphormer pattern \citep{ying2021graphormer} over the CSP constraint graph),
  normalized-coordinate features, and value-slot padding remove all size-specific
  parameters, so one checkpoint trains on 4$\times$4 and 6$\times$6 and transfers zero-shot
  to 9$\times$9 (the local propagator only; the global conflict and policy heads do not
  transfer, \S\ref{sec:results}); a frozen, identical-data, equal-budget protocol with GRAM
  and LDT baselines from our reimplementations and published failure criteria (the GRAM
  side is budget-limited and scoped to the soundness axis, \S\ref{sec:protocol}); and a closing
  claim audit mapping each causal claim in the paper to the measurement that supports it.
\end{enumerate}

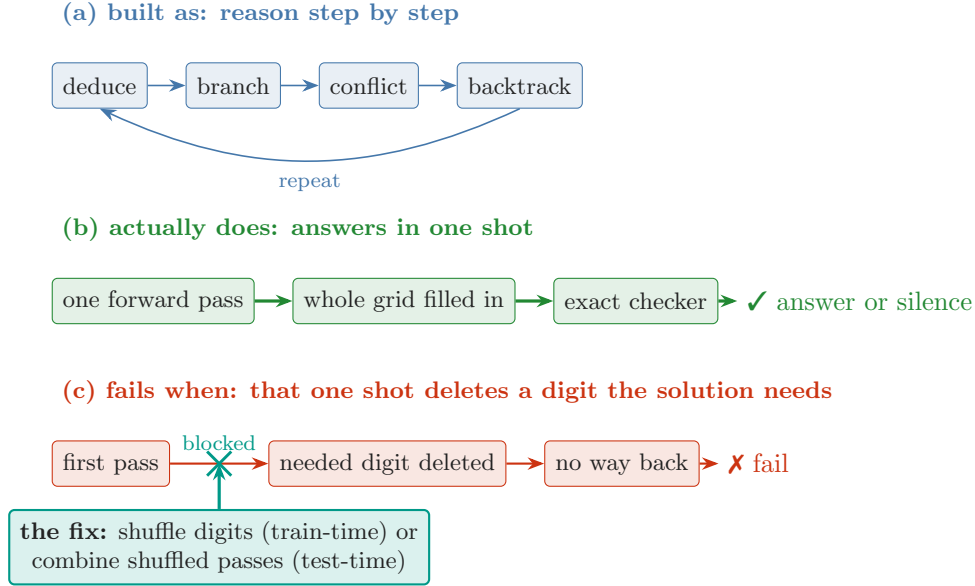
\begin{figure}[t]
\centering
\definecolor{cIntent}{HTML}{4477AA}  
\definecolor{cActual}{HTML}{228833}  
\definecolor{cFail}{HTML}{CC3311}    
\definecolor{cCure}{HTML}{009988}    
\definecolor{cInk}{HTML}{222222}
\begin{tikzpicture}[
  font=\footnotesize,
  box/.style={draw, rounded corners=2pt, inner sep=4pt, minimum height=6mm, text=cInk},
  ibox/.style={box, draw=cIntent, fill=cIntent!12},
  abox/.style={box, draw=cActual, fill=cActual!12},
  fbox/.style={box, draw=cFail, fill=cFail!12},
  cbox/.style={box, draw=cCure, fill=cCure!14, thick},
  lbl/.style={font=\footnotesize\bfseries},
  arr/.style={-{Stealth[length=2.2mm]}, semithick, color=cInk},
  iarr/.style={arr, color=cIntent},
  aarr/.style={arr, color=cActual, very thick},
  farr/.style={arr, color=cFail, thick},
  carr/.style={arr, color=cCure, very thick},
  node distance=3.5mm and 5mm]

\node[lbl, text=cIntent] (la) {(a) built as: reason step by step};
\node[ibox, below=of la.south west, anchor=north west] (d1) {deduce};
\node[ibox, right=of d1] (b1) {branch};
\node[ibox, right=of b1] (c1) {conflict};
\node[ibox, right=of c1] (k1) {backtrack};
\draw[iarr] (d1) -- (b1);
\draw[iarr] (b1) -- (c1);
\draw[iarr] (c1) -- (k1);
\draw[iarr] (k1.south) to[out=-155,in=-25] node[below=0.4mm, font=\scriptsize, text=cIntent, pos=0.5]{repeat} (d1.south);

\node[lbl, text=cActual, below=13mm of d1.south west, anchor=north west] (lb)
  {(b) actually does: answers in one shot};
\node[abox, below=of lb.south west, anchor=north west] (f2) {one forward pass};
\node[abox, right=of f2] (m2) {whole grid filled in};
\node[abox, right=of m2] (v2) {exact checker};
\node[right=2.5mm of v2, font=\small, text=cActual] (o2) {\ding{51} answer or silence};
\draw[aarr] (f2) -- (m2);
\draw[aarr] (m2) -- (v2);
\draw[aarr] (v2) -- (o2);

\node[lbl, text=cFail, below=6mm of f2.south west, anchor=north west] (lc)
  {(c) fails when: that one shot deletes a digit the solution needs};
\node[fbox, below=of lc.south west, anchor=north west] (f3) {first pass};
\node[fbox, right=13mm of f3] (p3) {needed digit deleted};
\node[fbox, right=of p3] (x3) {no way back};
\node[right=2.5mm of x3, font=\small, text=cFail] (o3) {\ding{55} fail};
\draw[farr] (f3) -- coordinate[midway] (mid3) (p3);
\draw[farr] (p3) -- (x3);
\draw[farr] (x3) -- (o3);
\node[cbox, below=6mm of mid3, anchor=north, align=center]
  (cure) {\textbf{the fix:} shuffle digits (train-time) or\\
          combine shuffled passes (test-time)};
\draw[carr] (cure.north) -- (mid3);
\draw[line width=1.1pt, color=cCure]
  ([xshift=-1.6mm,yshift=-1.6mm]mid3) -- ([xshift=1.6mm,yshift=1.6mm]mid3);
\draw[line width=1.1pt, color=cCure]
  ([xshift=-1.6mm,yshift=1.6mm]mid3) -- ([xshift=1.6mm,yshift=-1.6mm]mid3);
\node[font=\scriptsize, text=cCure, above=0.5mm of mid3] {blocked};
\end{tikzpicture}
\caption{The paper in one picture. (a) The architecture implements an iterative
deduce--branch--backtrack loop. (b) In the clue-rich regimes tested, one forward pass
commits essentially the entire grid, so the loop reduces to a one-shot predictor inside an
exact verifier. (c) All observed failures are decided in that first pass (first-pass
poisoning), which monotone search cannot undo; the two symmetry interventions predicted by
this diagnosis block that transition (\S\ref{sec:amortized}, \S\ref{sec:phase4}).}
\label{fig:mechanism}
\end{figure}

\section{Background}
\label{sec:background}

\paragraph{The lattice.} A problem with $C$ cells and $V$ values per cell is represented as a
multi-hot state $x \in \{0,1\}^{C \times V}$; $x_{i,v}=1$ means value $v$ is still viable at
cell $i$. Candidate sets ordered by inclusion form a finite lattice with meet
$a \meet b = \min(a,b)$ (elementwise); $\bottom$ is any state with an empty cell. Deduction
moves monotonically down the lattice: candidates are only ever removed. The abstraction
operator over a solution set $S$ is $\alpha(S)(i) = \{s(i) : s \in S\}$, and the training
target in state $x$ with known solutions $Y$ is
\begin{equation}
\hat y \;=\; x \,\meet\, \alpha(\{\, y \in Y : y \text{ consistent with } x \,\}),
\label{eq:alpha}
\end{equation}
which collapses to $\bottom$ exactly when no known solution survives in $x$; this is the
supervision signal for the conflict head \citep{davis2026ldt}.

\paragraph{The LDT loop.} One \emph{step}: forward the recurrent network on $x$; eliminate
alive candidates with $\sigma(b_{i,v}) < \theta_{\text{elim}}$; flag conflict if a cell is
empty or the conflict head fires; otherwise \emph{branch}: pin one multi-candidate cell to
a value sampled from $\mathrm{softmax}(b_i/\tau)$. One loop serves both training (a resident
pool of trajectories advances one step per optimizer step) and inference (multi-chain with
restarts). Iteration-averaged, the loss combines asymmetric BCE against $\hat y$ with
$w^{+}{=}4.0, w^{-}{=}0.5$ (false eliminations are 8$\times$ costlier, the
soundness-preserving bias), conflict BCE against $\mathbf{1}[\hat y = \bottom]$, and a
singleton cross-entropy.

\paragraph{What our reproductions added.} Our LDT reproduction
\citep{komissarov2026ltd_reimpl} surfaced a bootstrap failure absent from the paper: on full
datasets the paper-faithful loss settles into a \emph{keep-every-candidate} equilibrium
(trajectories die within one blind branch, the pool never develops depth, and the asymmetric
BCE makes never-eliminating locally optimal). The fix, a curriculum over deduction depth
that seeds the training pool with partially revealed states and anneals to clues-only, is
inherited unchanged by \colt{}.

\section{\colt{}}
\label{sec:method}

\colt{} keeps LDT's lattice, step operator, losses, pool training, and verify-or-abstain
emission, and changes three things: who chooses the branch cell, what happens on conflict, and
how structure enters the network.

\subsection{Branch-policy head}
\label{sec:policy}

A per-cell scalar logit $\pi_i$ is read at each unroll iteration alongside the candidate and
conflict heads. Its supervision is available for free from the $\alpha$-operator machinery:
define the \emph{branch-survival probability}
\begin{equation}
p^{*}_i \;=\; \sum_{d \,:\, \hat y_{i}[d] = 1} \mathrm{softmax}\!\big(b_i / \tau \,\big|\, \text{alive}\big)[d],
\label{eq:pstar}
\end{equation}
the probability that committing a value sampled from the \emph{actual} branch distribution at
cell $i$ keeps a known solution alive. The policy loss is
$\mathrm{BCE}(\sigma(\pi_i),\, p^{*}_i)$ over multi-candidate cells in non-$\bottom$ states,
averaged over iterations, with weight $\lambda_{\text{pol}}{=}0.25$ added to the LDT loss. At
inference, branching happens at the (Gumbel-sampled or greedy) argmax of $\pi$ over
multi-candidate cells.

The target $p^{*}$ can capture the candidate-count effect exploited by MRV, but is not
restricted to it: the policy head can additionally use learned value preferences and
constraint-graph context. MRV is retained as an independent hand-coded baseline
rather than treated as a special case of $p^{*}$; if the learned policy fails to beat it,
the component is falsified (\S\ref{sec:protocol}). The training pool itself branches with the learned policy
as it trains, keeping training states in-distribution for the search run at test time.

\subsection{DFS with chronological backtracking and a nogood ban set}
\label{sec:dfs}

Each inference chain maintains a stack of decision frames $(x_{\text{pre}}, i, q, T)$: the
pre-pin state, the branched cell, the branch distribution at decision time, and the set of
values already tried. On conflict, the chain \emph{backtracks}: pop to the deepest frame with
untried alive values, exclude the failed value, re-pin from the renormalized stored
distribution, and resume. This is chronological backtracking with value exclusion, i.e.\
depth-first search with the network as propagator and the policy head as variable-ordering
heuristic. Stack exhaustion triggers a fresh restart; chains stay diverse through sampled
value orderings.

A completed grid that fails verification is added to a per-puzzle \textbf{ban set shared by
all chains}; any chain that re-derives a banned grid treats it as an immediate conflict
(prune), skipping re-verification. Two soundness notes. First, a neural ``conflict'' is
evidence, not proof (an elimination can be wrong), so value exclusion is heuristic
pruning; this is safe because emission remains verify-gated. Second, the ban set \emph{is}
exact: a grid that failed the verifier is provably not a solution. What was LDT's dominant
waste, re-deriving known-wrong completions, becomes search progress.

\subsection{Constraint-graph conditioning: removing size-specific parameters}
\label{sec:graph}

All size-specific parameters are removed from the network:
\begin{itemize}
  \item \textbf{Relational attention bias.} Each attention layer adds a learned per-head
  scalar $\beta_h[\mathrm{rel}(i,j)]$ to the attention logits, where $\mathrm{rel}(i,j)$
  bit-packs \{same row, same column, same box\} (8 ids; id 7 $\Leftrightarrow$ $i{=}j$) plus
  two CLS ids. This is the graph-bias pattern of \citet{ying2021graphormer} applied to the
  CSP constraint graph; the bias is initialized to zero (plain attention) and learned.
  \item \textbf{Normalized coordinate features.} Cell content receives an MLP over
  $(r/n,\, c/n,\, (r \bmod b_r)/b_r,\, (c \bmod b_c)/b_c,\, 1/n)$; there are no learned
  tables indexed by absolute position.
  \item \textbf{Value-slot padding.} The lattice multi-hot is padded from the task's $V$ to a
  fixed $v_{\max}{=}9$; padded slots are permanently dead candidates, semantically consistent
  with the lattice. The candidate head emits $v_{\max}$ logits sliced back to $V$.
\end{itemize}
One set of weights \emph{accepts} any board size or box shape with at most
$v_{\max}$ values per cell; how far its competence actually extends is an empirical
question, answered in \S\ref{sec:results}: full accuracy on the trained 4$\times$4 and
6$\times$6 boards, zero-shot transfer to 9$\times$9 of the local propagator only. Multi-task
training keeps one
solve-pool per task and round-robins optimizer steps. Where HRM, TRM, GRAM, and LDT all train
one model per benchmark, \colt{} is one reasoner whose task is specified by its constraint
graph, and the graph interface extends beyond Sudoku to any binary-constraint CSP
(\S\ref{sec:coloring} demonstrates this on graph coloring).

\section{Experimental protocol (frozen before training)}
\label{sec:protocol}

Before any \colt{} training, the full protocol (datasets with pinned MD5s, budgets,
baselines, success criteria) was frozen in the repository (\texttt{BENCHMARKS.md}). We say
\emph{frozen} rather than pre-registered deliberately: the protocol was committed before
execution and the commit history is the audit trail, but no third-party timestamped registry
was used; the immutable Git tag \texttt{r14-2026-07-15} pins the state reported here, and
all tables derive
from committed raw JSONs in \texttt{results/} via the repository's table-generation
script. We summarize the protocol here.

\paragraph{Datasets.} Deterministic generated unique-solution Sudoku, byte-identical to the
ad3002/LTD baseline runs: \texttt{sudoku4} (4$\times$4, minimal clues, 1700/300 train/test),
\texttt{sudoku6} (6$\times$6, box 2$\times$3, 20 clues, 1020/180), and a held-out
\texttt{sudoku9\_small} used only for zero-shot transfer probes; the 9$\times$9/25-clue
Phase-4 split (\S\ref{sec:phase4}) comes from the same generator. Construction
(script \texttt{build\_sudoku\_dataset.py}, seeded): a random complete grid by randomized
backtracking; holes dug in random order subject to solution uniqueness verified exactly after
each removal; exact-duplicate puzzles rejected; a random train/test split of the unique
puzzles. All comparisons are internal to the frozen data; published Sudoku-Extreme numbers
are on a different corpus and are not comparable to any number here.

\paragraph{Leakage audit.} Exact-duplicate rejection does not address symmetry: we
audited all train/test pairs up to the full Sudoku symmetry group (digit
relabeling $\times$ row permutations within bands $\times$ band permutations $\times$ the
column analogues $\times$ transpose for square boxes), exactly, via a
digit-permutation-invariant hash with exhaustive enumeration of the cell subgroup and exact
verification of each hit (\texttt{scripts/leakage\_audit.py},
\texttt{results/leakage\_audit.json}).
At 4$\times$4 the space contains only 288 valid grids: 298/300 test solutions appear
verbatim in train and 300/300 up to symmetry, so the 4$\times$4 cells measure
implementation correctness only. At 6$\times$6 there are no exact
overlaps (0/180 puzzles, 0/180 solutions) but 180/180 test solutions are
symmetry-equivalent to some training solution (179/180 for the hard slice). The domain
itself forces this: 6$\times$6 Sudoku has on the order of fifty essentially different
grids, so \emph{no} 6$\times$6 split can be symmetry-disjoint and all 6$\times$6
generalization is within-symmetry-class by construction. At 9$\times$9, by contrast, the
splits are clean at all levels: zero overlap of puzzles, solutions, digit orbits, and full symmetry orbits
(3{,}359{,}232 cell transforms enumerated per grid) for both \texttt{sudoku9\_small} and the
Phase-4 training set. Class-level generalization therefore rests on the
9$\times$9 results, and the headline cure (below 1\% to 96.5\%, \S\ref{sec:phase4}) is
measured
on a split that provably shares no solution class with training. Digit-permutation
augmentation samples a subgroup of the symmetry group the audit already enumerates, so the
augmented training distribution stays inside audited train orbits and remains disjoint
from all test orbits (unit-tested in \texttt{tests/test\_colt.py}).

\paragraph{Metrics.} Accuracy (verified answers / puzzles), soundness (correct answers /
emitted answers, defined only when answers are emitted; see below), abstain rate,
rounds-per-solve percentiles (compute profile),
and search diagnostics (suppressed unsound completions, restarts, backtracks). For zero-shot
transfer: elimination precision/recall of a single forward pass on never-seen 9$\times$9
states, conflict-head AUC separating consistent from corrupted states, and the policy
survival margin $p^{*}(\text{argmax}_\pi) - p^{*}(\text{random})$.

\paragraph{Budgets.} 32 chains $\times$ 60 rounds per puzzle at both 4$\times$4 and
6$\times$6, equal to the budgets behind the published LDT baseline numbers, so comparisons
run at matched search compute on identical puzzles.

\paragraph{Statistical resolution.} Test splits have $n{=}180$ (300 at 4$\times$4), so a
measured 100\% carries a Wilson 95\% interval of $[97.9, 100.0]$, 49.4\% carries
$[42.2, 56.7]$, and 76.1\% carries $[69.4, 81.8]$; near 50\% the interval half-width
reaches $\pm 7.3$\,pp, so small differences between systems are not resolvable from one
split (no formal power analysis was performed). Comparisons between
inference arms run on identical puzzles and are therefore paired: where we report
arm equalities they are exact set equalities (the same puzzles solved, not equal
percentages), for which no test is needed. The 6$\times$6 headline, the hard-slice
anatomy and union experiments, the 9$\times$9 $\pm$augmentation comparison, and the
$c{=}3.0$ policy grid are each reported over three training seeds (the seed repetition
the revision protocol specified, E3, now executed). The 2$\times$2 interaction, the
zero-shot transfer probes, the crossing arms, and boundary densities other than
$c{=}3.0$ remain single-seed and are read as point estimates with the above intervals;
their repetition, the five-seed extension, and dataset-generation seeds stay in E3. Where we
report a null (the policy head, DFS accuracy),
the claim is bounded: no effect detectable at this resolution, with the anatomy of
\S\ref{sec:anatomy} supplying the mechanism for a true null. On soundness, the
verify-or-abstain gate enforces one invariant in all runs: \emph{no invalid
grid is ever emitted}. The soundness \emph{rate} (correct/emitted) is 100\% wherever a
system emits at all; where a row abstains on all puzzles the rate is $0/0$, undefined,
and the tables mark it ``---'' with the invariant holding vacuously. GRAM, which always
emits, is the one system whose soundness column is informative below 100\%.

\paragraph{What is being compared.} A plain exact backtracking solver with constraint
propagation closes all test sets used here, at a median of 0.017--0.21\,ms per
puzzle, with a median of \emph{zero} branching decisions on the 4$\times$4 and 6$\times$6
splits including the hard slice and a median of three at 9$\times$9 (measured;
\texttt{results/classical\_reference.json}); none of these benchmarks
is hard for classical search at these sizes. All comparisons are therefore \emph{internal}:
learned systems
against learned systems at matched neural budget, measuring how well sound amortized
deduction can be learned inside the verify-or-abstain envelope. MRV, the one classical
heuristic used as a baseline, enters as an inference arm over the learned propagator.

\paragraph{Baselines.} (i) \textbf{LDT}: our reproduction's published numbers on the same
data: 96.3\%/100\% accuracy/soundness at 4$\times$4 and 49.4\%/100\% at 6$\times$6;
additionally, the \texttt{restart\,$\times$\,random} arm of the \colt{} checkpoint reproduces
LDT inference exactly with matched training. The published LDT recipe includes
digit-permutation augmentation; our baseline omits it (the reproduction's documented top
deviation), so this row is an instrumented control under matched internal conditions rather
than a comparison with the full published recipe. (ii) \textbf{GRAM}: our reimplementation
on the
same data at a matched training budget. This row measures the soundness
axis: GRAM has no abstention mechanism, so its wrong-answer rate on emitted grids quantifies
what verify-or-abstain buys. GRAM's published
recipe uses 10--100$\times$ more compute than the matched budget allows, so the row carries
no information about GRAM's capability ceiling. (iii) \textbf{MRV}: the classical
variable-ordering heuristic as an
inference arm on the \colt{} checkpoint.

\paragraph{Frozen success criteria.} (1) some \colt{} arm exceeds 49.4\% on
\texttt{sudoku6} at the frozen budget; (2) learned $>$ MRV $>$ random on the DFS row;
(3) DFS $>$ restart for each policy; (4) soundness $=$ 1.0 in all cells (operationally:
zero invalid emissions; a cell that emits nothing satisfies the invariant vacuously and its
rate is reported as undefined); (5) the multi-size
checkpoint is within 5\,pp of the single-size checkpoint per board, with positive 9$\times$9
zero-shot transfer (elimination precision $>0.9$ at recall $>0.1$, AUC $>0.8$). Failures are
reported as failures: criteria (1)--(3) and (5) each falsify a specific design component
(\texttt{DESIGN.md} \S4). Criterion 5 is composite and is reported component by component;
it does not pass in full, because the global heads fail to transfer (\S\ref{sec:results}).

\section{Results}
\label{sec:results}

\emph{All numbers in this paper are produced by commands committed to the repository; the
raw result files live in \texttt{results/}. Each 6$\times$6 training run takes minutes
on a GPU and under an hour on CPU; the 9$\times$9 runs take a few hours on a 24\,GB
GPU.}

\subsection{Head-to-head at 6$\times$6 (frozen budget 32 chains $\times$ 60 rounds)}

\begin{table}[h]
\centering\footnotesize\setlength{\tabcolsep}{3pt}
\begin{tabular}{llrrrr}
\toprule
\textbf{system} & \textbf{search $\times$ policy} & \textbf{acc.\,\%} & \textbf{sound.\,\%} & \textbf{rounds p50} & \textbf{suppressed} \\
\midrule
LDT (repro, w/o its augmentation)$^{\ddagger}$ & restart $\times$ random & 49.4 & 100 & 3 & 148{,}567 \\
\colt{} & restart $\times$ random & 100.0 & 100 & 1 & 0 \\
\colt{} & restart $\times$ mrv     & 100.0 & 100 & 1 & 0 \\
\colt{} & restart $\times$ learned & 100.0 & 100 & 1 & 0 \\
\colt{} & dfs $\times$ random      & 100.0 & 100 & 1 & 0 \\
\colt{} & dfs $\times$ mrv         & 100.0 & 100 & 1 & 0 \\
\colt{} & \textbf{dfs $\times$ learned (full \colt{})} & \textbf{100.0} & \textbf{100} & \textbf{1} & \textbf{0} \\
\bottomrule
\end{tabular}
\caption{6$\times$6 ablation grid: one \colt{} checkpoint, inference-side switches; identical
puzzles and budget in all rows. The 6$\times$6 domain admits no symmetry-disjoint split
(\S\ref{sec:protocol}), so this table diagnoses mechanisms within symmetry classes, and
class-level generalization rests on \S\ref{sec:phase4}.
$^{\ddagger}$The LDT row is the published baseline of
\citep{komissarov2026ltd_reimpl} (its own checkpoint, same data and budget), trained
\emph{without} the LDT recipe's digit-permutation augmentation (the reproduction's
documented top deviation); see \S\ref{sec:amortized} for the 2$\times$2 that addresses this
directly. \colt{} rows are the canonical checkpoint (same recipe and seed as the original
run, retrained and evaluated in the single reporting environment of \S\ref{sec:anatomy});
it answers 180/180 with nothing suppressed. Counts and Wilson 95\% intervals for the
contrast: \colt{} 180/180 $[97.9, 100.0]$, LDT
89/180 $[42.2, 56.7]$; the gap is resolved at this $n$, and the equalities along the
\colt{} rows are exact (identical solved sets, not rounding).}
\label{tab:grid6}
\end{table}

\paragraph{Emission validity of a no-abstention baseline.} GRAM, our reimplementation at
the same matched budget on the same data, has no abstention mechanism: it emits a grid for
each of the 180 puzzles, and 98.9--100\% of those grids are wrong (2/180 and 0/180 valid
over two evaluation runs of the same checkpoint; emission is stochastic, both artifacts
committed). The number quantifies what the verify-or-abstain gate buys and nothing else:
GRAM is budget-starved by 10--100$\times$ relative to its published recipe
(\S\ref{sec:protocol}), so the comparison carries no information about its capability
ceiling.

Across training seeds 42, 43, and 44 the full-\colt{} cell
scores 100.0\%, 98.9\%, and 99.4\%, i.e.\ \textbf{99.4 $\pm$ 0.6\,\%} (mean $\pm$ std), all
seeds at 100\% soundness (the per-seed values are the F rows of Table~\ref{tab:e8}, one
environment). Two frozen success criteria resolve immediately: \colt{} exceeds
the baseline by \textbf{+50.6\,pp} at the matched seed (criterion 1), and soundness is
100\% in all cells (criterion 4). Criterion 3 (DFS $>$ restart) has nothing to measure
here: the benchmark is \emph{saturated}, median rounds-to-solve is 1, all six arms solve
all 180 puzzles with zero suppressed completions, and search is rarely invoked at all; its
mechanism metric resolves on the hard slice (\S\ref{sec:hard}).
Relative to the matched positional-table control within the \colt{} codebase, the training
architecture varies along exactly three components: the relational constraint-graph bias,
the coordinate MLP, and the policy-loss term (curriculum, data, optimizer, and scale are
matched). On probe curves, the \colt{} training run reached 100\% on a held-out probe by
step 1{,}000 of 5{,}000; the historical LDT baseline (a separate implementation) sat at
0\% at the same step on the same data and reached 58\% by step 4{,}000. The inference
ablations show that the search-side changes do not explain \colt{}'s accuracy; the
pre-specified single-component ablation (E8; six arms, all
trained and evaluated in one environment) then isolates the training side
(Table~\ref{tab:e8}).

\begin{table}[h]
\centering\footnotesize\setlength{\tabcolsep}{3.5pt}
\begin{tabular}{llcccll}
\toprule
\textbf{arm} & & \textbf{graph bias} & \textbf{coord MLP} & \textbf{policy loss} & \textbf{std\,\%} & \textbf{hard\,\%} \\
\midrule
A & positional tables (LDT-style) & --- & --- & --- & 0.0 (0--0) & 0.0 (0--0) \\
B & graph bias only & \checkmark & --- & --- & \textbf{99.4} (99.4--99.4) & \textbf{78.1} (76.7--80.6) \\
C & coord MLP only & --- & \checkmark & --- & 0.0 (0--0) & 0.0 (0--0) \\
D & \colt{} minus policy loss & \checkmark & \checkmark & --- & 99.6 (99.4--100.0) & 78.5 (76.1--80.6) \\
E & \colt{} minus coord MLP & \checkmark & --- & \checkmark & 99.3 (98.9--99.4) & 73.9 (69.4--76.1) \\
F & \colt{} full (retrain) & \checkmark & \checkmark & \checkmark & 99.4 (98.9--100.0) & 75.9 (72.8--78.9) \\
\bottomrule
\end{tabular}
\caption{E8 single-component training ablation over three training seeds (42/43/44),
reported as mean (min--max): identical data, curriculum, optimizer, and
budgets; \texttt{dfs\,$\times$\,learned} where the policy loss exists, else
\texttt{dfs\,$\times$\,random}; one CPU environment
(\texttt{results/ablate6\_*.json}). Held-out probe curves match: all graph-bias arms
reach probe 1.0 by step 1{,}000 on all seeds; arms A and C never leave 0.}
\label{tab:e8}
\end{table}

Seed-stable across all three runs, the result meets the pre-specified criterion
H-E8-graph: the constraint-graph attention bias alone accounts for the entire
positional-table-to-full-\colt{} gap (per-seed fractions 0.994, 1.006, 1.000; the
graph-bias-only arm is indistinguishable from full \colt{} on the standard slice), and the
coordinate MLP alone recovers nothing on any seed. Within this implementation,
constraint-graph conditioning is the
load-bearing training-side component, sufficient on its own to recover full-\colt{}
performance. Two caveats.
First, the positional-table control collapses to 0.0\% inside this codebase, below the
49.4\% of the historical LDT baseline (a separate implementation whose details differ
beyond the three components); the component attribution is therefore within-codebase, and
the cross-codebase headline row of Table~\ref{tab:grid6} stands unchanged with its own
caveats. A budget sweep substantially narrows that discrepancy and rules out an absolute
capacity failure: the positional-table
control stays at 0\% across learning rates at the frozen 5{,}000 steps, reaches 16.1\% at
3$\times$ the steps, and averages 37.2\% (28.3--47.8\%, seeds 42/43/44) at 10$\times$
(\texttt{results/armA\_*.json}, \texttt{budget\_sweep\_pos50k\_*.json}), one seed
approaching the historical baseline's 49.4\%. The frozen-budget advantage of graph bias is
therefore primarily an optimization and sample-efficiency effect, per the pre-specified
decision rule, although a residual cross-implementation gap remains unexplained. Second, at seed 42 the two arms \emph{without} the policy loss led full \colt{}
by $+4.5$\,pp on the hard slice, suggesting the policy loss miscalibrates the
propagator; the trend did not survive seed repetition (seed 44 reverses it), and across
three seeds the four graph-bias arms are indistinguishable on the hard slice (means
73.9--78.5\% with per-seed swings of the same size).

\subsection{A harder slice: search becomes active, accuracy still propagator-bound}
\label{sec:hard}

To unsaturate the search layer we built a zero-shot hard slice: 180 unseen 6$\times$6 puzzles
at 14 clues (vs.\ 20 in training; deeper propagation chains, branching required), evaluated
with the same checkpoint and budget.

\begin{table}[h]
\centering\small
\begin{tabular}{lrrrr}
\toprule
\textbf{arm} & \textbf{acc.\,\%} & \textbf{sound.\,\%} & \textbf{suppressed} & \textbf{backtracks} \\
\midrule
restart $\times$ \{random, mrv, learned\} & 76.1 & 100 & 74{,}864 & 0 \\
dfs $\times$ \{random, mrv, learned\}     & 76.1 & 100 & \textbf{50} & 5{,}120 \\
\bottomrule
\end{tabular}
\caption{Hard slice (6$\times$6, 14 clues, zero-shot; canonical checkpoint;
within-symmetry-class generalization, as in Table~\ref{tab:grid6}). Backtracking +
the nogood ban set cut
wasted wrong-completion derivations by \textbf{1{,}497$\times$} at identical accuracy and
identical budget. A budget sweep (5, 15, 60 rounds) leaves all accuracies unchanged at 76.1\%:
solvable puzzles fall within 5 rounds, the rest resist all arms.}
\label{tab:hard}
\end{table}

Again the grid is flat in accuracy, and the budget sweep rules out a budget
artifact: the puzzle population \emph{bifurcates}. Either propagation (plus trivial search)
solves a puzzle almost immediately, or no search arm rescues it within any tested budget. The
policy criterion (learned $>$ MRV $>$ random, criterion 2) is therefore \emph{not testable at
6$\times$6}: the policy is never load-bearing.

\subsection{Failure anatomy: first-pass poisoning explains the failures}
\label{sec:anatomy}

Monotonicity makes one direction of the diagnosis a lemma rather than a finding: once a
true-solution value is eliminated, no later monotone step can restore it, so first-pass
poisoning is \emph{sufficient} for failure. The empirical question is whether it is also
\emph{necessary}: does any puzzle survive its first pass clean and still fail later, from
branching, conflict handling, or budget exhaustion? For each test puzzle we run \emph{one}
forward pass on the clues-only lattice and
check whether threshold elimination removes any true-solution value (``first-pass
poisoning''). Eliminations are deterministic per state, so a poisoned first pass is inherited
by all chains of all arms. The contingency on the hard slice answers the necessity
question:

\begin{center}
\small
\begin{tabular}{lrr}
\toprule
 & solved (137) & failed (43) \\
\midrule
first pass clean    & \textbf{137} & 0 \\
first pass poisoned & 0 & \textbf{43} \\
\bottomrule
\end{tabular}
\end{center}

All 43 failures, and only the failures, are first-pass poisoned
(\texttt{results/anatomy6hard\_cpu.json}); the two probe code paths that measure it are
bitwise identical on this checkpoint
(\texttt{results/\allowbreak reconcile\_anatomy\_h2.json}). Replication: seeds 43 and 44 of the same recipe reproduce it at 49/49 vs.\ 0/131 and
38/38 vs.\ 0/142 (\texttt{results/anatomy6hard\_cpu\_s4*.json}), and an earlier run
on GPU hardware shows it at
42/42 vs.\ 0/138 (\texttt{results/anatomy6hard.json}); the union cure closes all three
seeds to 100\% (residual poisoning 1.1--2.2\%,
\texttt{results/h2\_colt6cpu\_s4*\_union.json}). The zero in the failed-clean cell
is the finding: no late-failure class is detectable on this slice, every failure is decided
in the first pass, and the causal weight rests on the interventions of
\S\ref{sec:amortized} and \S\ref{sec:phase4}, which the contingency predicts. It also explains
the flat ablation grids,
locates the accuracy ceiling in propagator \emph{calibration}, and
retroactively rationalizes a design choice in the original LDT paper that we had skipped as an
unimplemented deviation: per-step randomized augmentation (digit permutation + dihedral
symmetry applied and inverted around each forward pass), which is a mechanism for
de-correlating elimination errors across steps and chains. We tested the cheapest calibration
lever, lowering $\theta_{\text{elim}}$ at inference, and it has no effect: the
poisoned count is 43 at $\theta_{\text{elim}} \in \{0.10, 0.05, 0.02\}$ alike, and accuracy
stays at 76.1\% (\texttt{results/reconcile\_anatomy\_h2.json}; the GPU-era run shows the
same invariance, 42 at all three thresholds). The propagator is not borderline on
these values; it \emph{confidently} zeroes them ($\sigma(b) < 0.02$). Calibration must
therefore come from training-side or ensemble mechanisms, and the next subsection and
\S\ref{sec:phase4} test the two cheapest directly: test-time elimination ensembling across
symmetry frames, and train-time randomized digit-permutation augmentation.

\subsection{One-shot amortization, its two cures, and the structure-augmentation interaction}
\label{sec:amortized}

Three follow-up probes sharpen the anatomy into an architectural diagnosis.

\paragraph{H1: the model is a one-shot solver.} On the standard 6$\times$6 test split, a
\emph{single} forward pass plus threshold elimination drives \textbf{all blank cells of all 180 puzzles} to exactly one candidate (singleton rate 1.000, full-commit fraction 1.000;
\texttt{results/h1\_colt6cpu.json}); on the zero-shot hard slice the rates are 0.998 and 0.956,
and 5.9\% of committed singletons are wrong: these are the poisoned puzzles. In these
clue-rich regimes the
``iterative deduction'' loop is cosmetic at test time: the candidate head amortizes the entire
solution into its first elimination pass, and the lattice machinery functions as a sound
verifier around a one-shot guesser. This is a property of the regime, not of the
architecture (\S\ref{sec:boundary} shows the same model with a singleton rate of 1.4\% on
from-scratch coloring, where the loop genuinely iterates). It explains the flat ablation
grids here: search over branch values cannot matter when there is nothing left to branch on.

The objective targets a one-shot solution, and the trained model nearly realizes that
target. For a unique-solution puzzle the depth-0 target of
Eq.~\ref{eq:alpha} \emph{is} the complete solution ($\alpha$ of a singleton solution set
keeps exactly the solution values), so a candidate head that fits the objective perfectly
is, by construction, a one-shot solver; H1 measures how completely that optimum is reached,
and its diagnostic value is to rule out the iterative-deduction reading of the architecture
at test time. Supervision that targets
bounded-depth consequences, or the multi-solution sets of \S\ref{sec:boundary}, changes the
optimum itself and is the principled route to a genuinely iterative propagator; a
depth-limited objective is specified in the revision protocol. The
augmentation-trained 9$\times$9 models remain 94--96\% one-shot but with commit precision
\textbf{0.998--0.999} across seeds (vs.\ 0.584--0.589 unaugmented): augmentation does not
change the amortized nature
of the computation, it calibrates it, and the few puzzles it no longer fully
commits in one pass are where genuine multi-step iteration first appears.

\paragraph{H2: test-time symmetry decorrelation cures poisoning without retraining.} If
poisoning is value-specific bias of a deterministic propagator, then running each forward
pass in $K$ random digit-permutation frames (permute the value axis, forward, invert) and
\emph{keeping any candidate that any frame keeps} (union aggregation) should de-poison the
first pass at $K\times$ compute and zero retraining. It does
(\texttt{results/h2\_colt6cpu\_hard\_union.json}, unaugmented canonical checkpoint, hard
slice):

\begin{center}
\small
\begin{tabular}{lrrr}
\toprule
 & $K{=}1$ & $K{=}4$ union & $K{=}8$ union \\
\midrule
first-pass poisoning rate & 23.9\,\% & 5.0\,\% & 1.7\,\% \\
solve accuracy (frozen budget) & 76.1\,\% & \textbf{100.0\,\%} & \textbf{100.0\,\%} \\
\bottomrule
\end{tabular}
\end{center}

The $K{=}1$ column repeats \S\ref{sec:anatomy} exactly (43 poisoned, 76.1\% solved): all
of \S\ref{sec:results}--\ref{sec:amortized} is measured on one checkpoint in one
environment, and the two first-pass probes are bitwise identical on it
(\texttt{results/reconcile\_anatomy\_h2.json}). GPU-era runs of the same recipe replicate
the pattern within 3--4 margin-borderline puzzles (42--46 poisoned, 74.4--76.7\% solved;
committed artifacts), with the dichotomy exact in each environment.

Mean aggregation reaches only 97.8\,\% at $K{=}8$ (poisoning 16.1\,\%): the \emph{union} is
the active ingredient, as the anti-poisoning design predicts. A frame-type factorial adds
the predicted null and a specificity test (\texttt{results/h2\_factorial\_cpu.json}, same
checkpoint): identity frames (the same pass repeated $K$ times) leave all 43 poisonings in
place at $K \in \{4, 8\}$; at $K{=}8$, geometry-only frames (row/band/column/stack
permutations, no
digit relabeling) reduce them only to 27 (solve 163/180), digit frames reach 3
(solve 180/180, the table above), and digit$\times$geometry composed reach 5 (solve
180/180). The bias is predominantly
value-specific, as claimed; geometry transformations also decorrelate a smaller subset of
errors, but digit relabeling is the dominant factor, and the wrapper itself does nothing
without a symmetry inside it. The 1.7\,\%
residual poisoning coexists with 100\,\% accuracy because frames are redrawn on each pass:
under the wrapper, poisoning stops being a deterministic property of the puzzle that all
chains inherit (\S\ref{sec:anatomy}); a later pass or another chain recovers the value, and
the verifier does the rest. Together with the
train-time result (\S\ref{sec:phase4}), the causal chain (deterministic value-specific
confident mis-elimination; decorrelate by value relabeling; the failure mode is removed) is
confirmed by complementary train-time and test-time interventions acting on the same
value-symmetry mechanism, on two boards, one requiring no retraining. This
also mechanistically explains the original LDT recipe's per-step inference augmentation
\citep{davis2026ldt}.

\paragraph{The parameterization--augmentation interaction at the frozen budget.} A 2$\times$2 factorial
separates the parameterization from the augmentation (6$\times$6, identical data, budget,
and training scale in all four cells, run together in the original GPU environment):

\begin{center}
\small
\begin{tabular}{lrr}
\toprule
accuracy \% & $-$augmentation & $+$augmentation \\
\midrule
LDT (positional tables) & 49.4 & \textbf{0.0} \\
\colt{} (relational parameterization) & 98.9 & \textbf{100.0} \\
\bottomrule
\end{tabular}
\end{center}

At 5{,}000 steps, augmentation reduces the positional model to zero accuracy (it must
learn value-specific per-position patterns, which relabeling erases faster than the budget
can absorb; it emits nothing, so the verified-emission invariant holds vacuously) but does
not reduce the relational model's accuracy:
its relational computation is compatible with value relabeling. The unaugmented budget
sweep of \S\ref{sec:results} (positional tables reaching 37\% at 10$\times$ steps)
indicates this is a convergence-speed interaction rather than an absolute incompatibility;
the augmented extended-budget arm (1$\times$, 3$\times$, 10$\times$ steps, E1) remains
in the revision protocol. At the frozen budget the practical reading stands: the \colt{}
parameterization is what makes augmentation affordable there. ``The \colt{}
parameterization'' here bundles the three training-side components; the interaction is
established for the bundle, and the E8 ablation (Table~\ref{tab:e8}) identifies the graph
bias as the load-bearing component of that bundle in the unaugmented column. Two caveats.
First, the LDT row throughout this paper omits the LDT recipe's own augmentation (its
documented top deviation in our reproduction), so the headline gap (+49.5\,pp in the
2$\times$2's GPU environment; +50.6\,pp for the canonical checkpoint of
Table~\ref{tab:grid6}) measures \colt{} against
LDT-minus-its-calibration-mechanism; the decomposition above, not the gap, is the finding.
Second, this 2$\times$2 exists at 6$\times$6 only; the 9$\times$9 factorial is
in the frozen revision protocol.

There is also an architectural alternative that the augmentation result points at directly:
if poisoning is value-specific bias, exact \emph{value-permutation equivariance} (weight
sharing and pooling along the value axis, the analogue of NeuroSAT's literal-permutation
invariance \citep{selsam2019neurosat}) would remove the biased directions structurally and
make the $V!$-fold augmentation factor unnecessary; augmentation teaches the identical
symmetry statistically. Whether exact equivariance is free on clue-rich tasks, where the clue digits
must still break the symmetry through context, is an open empirical question, frozen in the
revision protocol as the completion of this ablation.

\FloatBarrier
\subsection{One checkpoint across boards}

\begin{table}[h]
\centering\small
\begin{tabular}{lrr}
\toprule
 & \textbf{4$\times$4 acc.\,\%} & \textbf{6$\times$6 acc.\,\%} \\
\midrule
single-size checkpoints & 100.0 & 98.9 \\
multi-size checkpoint (one set of weights) & \textbf{100.0} & \textbf{99.4} \\
\bottomrule
\end{tabular}
\caption{The one-checkpoint claim (soundness 100\% in all cells, dfs $\times$ learned at the
frozen budgets; both rows from the original GPU environment, internally matched; the
canonical retrain of the single-size checkpoint scores 100.0, Table~\ref{tab:grid6}). Multi-task training over the constraint-graph interface costs no accuracy:
the 6$\times$6 cell is indistinguishable from the single-size checkpoint (99.4 vs.\ 98.9 is
one puzzle in 180), passing the no-degradation component of frozen criterion 5; the local
elimination-transfer component also passes, whereas the global-head transfer component
fails (AUC 0.50 and an environment-sensitive policy margin). Positive transfer is below
the resolution of this split. A same-recipe CPU retrain of the multi-size checkpoint
reproduces both cells at 100.0/100.0 (\texttt{results/multi\_cpu\_*.json}).}
\label{tab:multi}
\end{table}

\paragraph{Zero-shot 9$\times$9 transfer.} On 512 mixed-depth states from 9$\times$9 puzzles
the multi-size checkpoint has \emph{never seen any board of this size}, one forward pass
eliminates with \textbf{precision 0.977 at recall 0.597}, above the
frozen bar (precision $>0.9$ at recall $>0.1$). A base-rate control cuts this number down
to size: plain constraint propagation from the decided cells already accounts for 98.3\%
of the true eliminations on these states, and restricted to candidates propagation cannot
eliminate, the model's precision falls to 0.33 at recall 0.45
(\texttt{results/transfer9\_floor.json}). What transfers is therefore chiefly a learned
reimplementation of constraint propagation; the lift beyond it is weak and noisy. The two
\emph{global} heads do not transfer even that:
conflict-head AUC is 0.50 (chance) in both measurement environments; the policy survival
margin is environment-sensitive ($-0.03$ on the GPU-era checkpoint, $+0.11$ on a
same-recipe retrain) and we read it as noise
(\texttt{results/transfer9\_multi*.json}).
Elimination is a local computation over the constraint neighborhood, which the graph bias
makes size-invariant, while conflict detection and branch scoring integrate global state whose
statistics shift from 36 to 81 cells. Making the global heads size-invariant (e.g.\
constraint-graph pooling instead of a CLS token) is a concrete architectural follow-up.

\subsection{Where the compute goes}

At 6$\times$6, \colt{} solves with median \textbf{1} forward pass per puzzle (vs.\ 3 for the
LDT baseline at less than half the accuracy). The standard slice leaves nothing to cut
(zero suppressed completions in all six arms); on the hard slice the DFS arms reduce wasted
wrong-completion derivations by three orders of magnitude
(74{,}864\,$\to$\,50). Under the verify-or-abstain regime this waste is
the quantity a deployment pays for (each suppressed completion is a full
verification call plus the forwards that produced it), so the nogood/backtracking machinery
is a cost result even where it is accuracy-neutral.

\section{9$\times$9 at 25 clues: the regime wall, and augmentation as the lever}
\label{sec:phase4}

Frozen in the repository before the runs (\texttt{PHASE4.md}), the 9$\times$9 head-to-head
pits \colt{} against our LDT reimplementation, both trained on one frozen
generated split at 25 clues (31\% cell coverage, against 55\% for the 6$\times$6 benchmark)
and evaluated at 64 chains $\times$ 200 rounds, on 24\,GB GPUs (RTX 4090 for the
original run, RTX 3090 for the three-seed replication).

\paragraph{The wall.} Without augmentation, \colt{} at 20k steps solves 0/180, 0/180, and
1/180 across three training seeds: the first pass poisons 180, 180, and 179 of the 180
test puzzles (the contingency is total), and the restart solver suppresses 0.5--1.7M wrong
completions per seed without emitting one, where DFS with the ban set suppresses only
41--137 (a four-orders-of-magnitude reduction in wasted verification even in failure mode).
The single-forward commit precision explains the wall: the unaugmented model still commits
almost every cell in one pass (singleton rate 0.97--1.00) but only 58.4--58.9\% of those
commitments are correct. Trained on the same split, the LDT
reimplementation scores the same \textbf{0\,\%} (suppressing 1.84M
wrong completions; single seed): without augmentation, \emph{neither} architecture clears
this regime.
The 6$\times$6 architectural gap does not transfer; the wall is shared.

\paragraph{Why, and the fix the anatomy predicts.} \S\ref{sec:anatomy} located the ceiling in
deterministic, \emph{confident} first-pass mis-eliminations. LDT's published recipe contains
the mechanism our reproduction had skipped (and flagged as its top deviation):
dataset-level digit-permutation $\times$ dihedral augmentation ($\sim$2880$\times$ expansion)
plus per-step augmentation at inference. Value relabeling is exact for Sudoku, and a
propagator trained across relabelings cannot keep value-specific elimination biases;
augmentation is calibration. We therefore trained \colt{}-aug, identical except
for a fresh random digit permutation per sampled pool trajectory. The result:

\begin{center}
\small\setlength{\tabcolsep}{3.5pt}
\setlength{\tabcolsep}{3pt}\footnotesize
\begin{tabular}{lrrrrr}
\toprule
\textbf{9$\times$9, 25 clues (64$\times$200; 3 seeds)} & \textbf{acc.\,\%} & \textbf{emitted} & \textbf{sound.\,\%} & \textbf{p50} & \textbf{suppressed} \\
\midrule
\colt{} (no aug), restart $\times$ random & 0.2 (0--0.6) & 0--1/180 & --- & --- & 0.5M--1.7M \\
\colt{} (no aug), dfs $\times$ learned    & 0.2 (0--0.6) & 0--1/180 & --- & --- & 41--137 \\
LDT reimpl.\ (no aug; batch 256, 1 seed) & 0.0 & 0/180 & --- & --- & 1{,}841{,}396 \\
\colt{}-aug, dfs $\times$ learned & \textbf{96.5} (96.1--96.7) & 173--174/180 & \textbf{100} & \textbf{1} & \textbf{0} \\
\colt{}-aug, restart $\times$ random & 96.5 (96.1--96.7) & 173--174/180 & 100 & 1 & 0 \\
\bottomrule
\end{tabular}

\smallskip
{\footnotesize \colt{} rows: mean (min--max) over training seeds 42/43/44, one GPU
environment (\texttt{results/colt9\_*\_s4*.json}). All-abstain cells emit no grids, so
the soundness rate is $0/0$ (undefined, shown ---); the verified-emission invariant holds
vacuously (\S\ref{sec:protocol}). The single grid emitted by the no-aug arm (seed 44) is
correct.}
\end{center}

A single exact data symmetry takes the same architecture, base dataset, and training
budget from at most 0.6\% to \textbf{96.5\% (96.1--96.7 across three seeds)} at 100\%
soundness, with a median of \emph{one} forward pass per solved puzzle
and \emph{zero} suppressed wrong completions: the 9$\times$9 wall was first-pass
poisoning, and decorrelating
value-specific elimination bias removes it. Augmentation leaves the computation one-shot
(singleton rate 0.94--0.96) and calibrates it: commit precision rises from 0.584--0.589 to
0.998--0.999. The residual failures close under the frozen probes: on each seed every
remaining failure is again first-pass poisoned (6/6, 7/7, 6/6 vs.\ zero clean failures),
and stacking $K{=}8$ union frames on the augmented model removes them, 100.0\% at residual
poisoning 0.6\% (\texttt{results/h2\_colt9aug\_union.json}); the train-time and test-time
symmetry cures compose. (Two notes: the LDT paper's full recipe
\emph{includes} this augmentation, so the LDT row above (our reimplementation without it,
trained at batch 256 after the paper-recipe batch 512 OOMs a 24\,GB card under full-unroll
BPTT) is a same-data control for our deviation, not a refutation of their recipe; and with
search again almost never invoked at 96\%, the branch policy remains untested in a regime with long sound search chains.)

\section{Beyond Sudoku: graph coloring through the same interface}
\label{sec:coloring}

\colt{}'s core (the lattice network, losses, solver loop, and training procedure) is
task-generic; each CSP supplies a task adapter defining the relation vocabulary, coordinate
features, exact verifier, and solution enumeration or sampling. We verify this division on
a second CSP instantiated through the identical core code
path: proper 4-coloring of a fixed
connected Erd\H{o}s--R\'enyi graph $G(20, 0.25)$. The constraint graph is irregular rather
than grid-structured, the relation vocabulary collapses to \{none, shares-an-edge, self\},
coordinate features reduce to normalized vertex degree and index, and puzzles
have \emph{multiple} valid completions (up to 16, enumerated exactly), exercising the
$K{>}1$ path of the abstraction operator for the first time: the target
$\hat y = x \meet \alpha(\cdot)$ is a genuine union over surviving solutions that sharpens as
the state commits. Puzzles are partial colorings with 30--55\% of vertices pre-colored; 1020
train / 180 test; the model, optimizer, loss, curriculum, and solver are byte-identical to
the Sudoku runs (d $=$ 64, 4 layers, $L{=}8$, 3{,}000 steps, under four minutes on the same
GPU). Train and test share this one fixed graph, so \S\ref{sec:coloring} is an interface
demonstration; distribution-over-graphs generalization is tested in \S\ref{sec:boundary},
where each instance carries its own random graph.

Result: \textbf{100\% accuracy, 100\% soundness, zero suppressed completions}, and each
emitted coloring lies inside the exactly-enumerated solution set of its puzzle
(\texttt{results/coloring20.json}): the multi-solution soundness contract holds as
specified. Graph coloring at this size is easy for classical search (\S\ref{sec:protocol});
the experiment shows that the architecture, training procedure, and soundness machinery
transfer to a structurally different CSP with multi-solution supervision through the task
adapter alone.

\section{Probing the capability boundary: 3-coloring across the phase transition}
\label{sec:boundary}

Solved task scores locate no boundary. To map one, we need a hardness dial with a
well-studied hard region, and random-graph 3-colorability provides the canonical one:
colorability of $G(n, m{=}cn/2)$ has a physics-predicted \emph{asymptotic} threshold near
average degree
$c \approx 4.69$ \citep{mulet2002coloring} (a cavity-method estimate for $n \to \infty$),
and instances near the transition are the classical
``really hard'' constraint problems
\citep{cheeseman1991hard,monasson1999determining}. We sweep
$c \in \{3.0, 3.25, 3.5, 3.75, 4.0, 4.5, 4.9\}$ at
$n{=}40$ (the quarter-step points were added to resolve the drop-off), training one model
per density on satisfiable instances (each with its own random
graph; the relation matrices are batched per sample) and solving \emph{from scratch}: zero
pre-colored vertices, the regime where completion becomes generation. Targets use up to
$K{=}12$ exactly-sampled solutions per instance. At $n{=}40$ the transition is
finite-size-smeared and sits below the asymptotic constant; our generator measures it
directly: the unsatisfiable fraction among generated graphs climbs
3\% $\to$ 4\% $\to$ 7\% $\to$ 14\% $\to$ 28\% $\to$ 76\% $\to$ 97\% across the sweep,
crossing 50\% between $c{=}4.0$ and
$c{=}4.5$. Hardness statements below refer to these measured fractions, not to the
asymptotic threshold.

\begin{center}
\small
\begin{tabular}{lrrrrrrr}
\toprule
average degree $c$ & 3.0 & 3.25 & 3.5 & 3.75 & 4.0 & 4.5 & 4.9 \\
\midrule
accuracy, restart $\times$ learned & 24.8 & 15.2 & 0.0 & 0.0 & 0.0 & 0.0 & 0.0 \\
accuracy, dfs $\times$ learned     & \textbf{36.2} & \textbf{27.6} & 0.0 & 0.0 & 0.0 & 0.0 & 0.0 \\
H1 first-pass singleton rate        & 0.014 & 0.000 & 0.0 & 0.0 & 0.0 & 0.0 & 0.0 \\
\bottomrule
\end{tabular}
\end{center}

Three findings. \textbf{First}, the from-scratch regime is the first in this study where
search changes accuracy: one-shot amortization (\S\ref{sec:amortized})
vanishes (singleton
rate 1.4\% at $c{=}3.0$ versus 100\% on Sudoku) and DFS beats restart at identical budget.
The sweep models show +11.4\,pp at $c{=}3.0$ and +12.4\,pp at $c{=}3.25$ (one seed each);
across the three policy-grid seeds at $c{=}3.0$ the policy-averaged DFS$-$restart delta is
+0.3/+4.5/+3.8\,pp, positive for every seed (pooled +2.9\,pp), so the direction is
seed-stable while the magnitude is not, and the single-seed sweep values sit at the
favorable end of that spread. The frozen expectation
that conflict-driven search matters where amortization fails is confirmed in direction on
both sides of the contrast. This regime also supplies the policy head's fair trial (criterion
2, untestable at 6$\times$6): over three training seeds and the full
\{restart, dfs\}$\times$\{random, MRV, learned\} grid at $c{=}3.0$, each informed
ordering beats random in 5 of the 6 search$\times$seed cells (MRV 5/6, learned 5/6; the
exception for both is dfs at seed 42), but the learned policy never
separates from hand-coded MRV (per-seed dfs accuracies 4.8/31.4/16.2\% learned vs.\
5.7/32.4/10.5\% MRV; \texttt{results/policy\_grid\_c30\_seed*.json}). Criterion 2 fails
in its strict form even where branching is load-bearing: ordering quality matters, the
learned head adds nothing over MRV. Training-seed variance in this regime is large (5--32\%
across seeds at fixed data), and all boundary numbers should be read with that spread. \textbf{Second}, the solvable region ends between $c{=}3.25$ and
$c{=}3.5$: accuracy declines 36.2\% $\to$ 27.6\% and then hits zero while the
unsatisfiable fraction is still 7\%, far below our measured 50\% crossover (between 4.0
and 4.5), let alone the
asymptotic threshold. Instances at all densities in the sweep are decided by a plain MRV
backtracking solver at a measured median of 0.66--1.31\,ms
(maximum 24\,ms, at $c{=}3.0$; \texttt{results/classical\_reference.json}), yet the
learned system solves none beyond $c{=}3.25$. The neural amortizer's boundary is not the classical solver's boundary;
instances trivial for backtracking are already unlearnable at this scale and budget.
\textbf{Third}, the failure is not a search failure: the deduction objective itself
provably carries no elimination signal where the work begins.

\begin{proposition}[No depth-0 elimination signal under value symmetry]
\label{prop:symmetry}
Let a CSP be value-symmetric (every permutation of the value alphabet maps solutions to
solutions; graph coloring is the canonical case) and satisfiable, with the same full value
alphabet available at every variable at depth zero, and let $\mathcal{S}$
denote its \emph{full} solution set. Then every (variable, value) pair participates in some
solution, so the depth-0 abstraction target computed from the full set,
$\hat y = \alpha(\mathcal{S})$, is all-ones.
\end{proposition}

\begin{proof}
Take any solution $s$ and any pair $(v, c)$. The transposition of $c$ with $s(v)$ is a value
permutation, hence maps $s$ to a solution $s'$ with $s'(v) = c$.
\end{proof}

\noindent\emph{Remark.} The exact target thus supplies no candidate-elimination labels at
depth zero: a head trained toward it has nothing to learn before the first commitment.

That proposition concerns the exact target; training in this section uses $K$
\emph{sampled} solutions $Y_K \subseteq \mathcal{S}$, and
$\alpha(Y_K) \subseteq \alpha(\mathcal{S})$: every (variable, value) pair the sample misses
becomes a spurious elimination label at positions where the exact signal keeps
everything. Finite sampling does not dilute the depth-0 signal, it
inverts it, a distinction the crossing arms below measure directly (denser sampling gives
the largest single-axis gain, and the scaled model's eliminations are noise-capped). Sudoku never
exposes any of this: clues break the value symmetry and seed forced structure, and its
unique-solution supervision uses the complete solution set, so its targets are exact.
From-scratch CSP is a \emph{generation} problem, and a deduction-supervised propagator has
nothing to amortize at its start. Proposition~\ref{prop:symmetry} generalizes to any
value-symmetric CSP solved from scratch, and it is the conceptual counterpart of
Observation~\ref{obs:licensing}: soundness frees the search layer, and value symmetry
removes the depth-0 deduction signal that amortization would need to start from. Where the wall sits relative to resources, and which
axis moves it (inference budget, training scale, supervision quality, value calibration), is
measured next.

\paragraph{Crossing attempts: denser supervision gives the largest gain among the tested one-axis interventions.}
Holding the task fixed at the first dead point ($c{=}4.0$, one seeded instance set), we move
one resource axis at a time (\texttt{results/boundary\_cross.json}; the density sweep is
\texttt{results/phase\_transition*.json}):

\begin{center}
\small
\begin{tabular}{llrr}
\toprule
arm & axis moved & restart \% & dfs \% \\
\midrule
A & none (baseline: $d{=}64$, 4k steps, $K{=}12$) & 0.0 & 0.0 \\
B & 12$\times$ inference budget (32$\times$240 vs.\ 16$\times$40), same checkpoint & 1.9 & 0.0 \\
G & learned conflict signal disabled & 0.0 & 0.0 \\
C & $\sim$10$\times$ training compute ($d{=}128$, 12k steps) & 9.5 & 8.6 \\
D & 4$\times$ supervision density ($K{=}48$ sampled solutions) & \textbf{15.2} & \textbf{14.3} \\
E & no learning (stub propagator, exact search only) & 0.0 & 0.0 \\
\bottomrule
\end{tabular}
\end{center}

So the boundary is resource-soft, but the axes are far from equivalent. Search-side resources
do almost nothing: 12$\times$ budget buys 1.9\%, recalibrating termination buys nothing, and
the no-learning control scores zero, so the $c{=}3.0$ successes came from the learned
propagator. Training-side resources are what move it:
4$\times$ denser solution sampling alone lifts accuracy from zero to 15.2\% at unchanged
model and budget, and 10$\times$ compute to 9.5\%. Elimination diagnostics explain the
ordering: the baseline propagator learns almost nothing actionable (elimination recall 1.2\%
at precision 0.93), while the scaled model eliminates aggressively but against
$K$-sampled targets whose noise caps its usable precision. Both findings are corollaries of
Proposition~\ref{prop:symmetry} and its sampling remark: the exact deduction signal exists
only in committed (deep) states, and the sampled target additionally injects false
elimination labels wherever the sample fails to cover the true solution set.
Among the interventions tested in this single-instance, single-seed probe, denser
solution sampling produced the largest gain (15.2\%), followed by increased training
compute (9.5\%), while additional inference search produced little improvement. This
points to supervision coverage as a major constraint in this regime, but does not
establish it as the unique binding factor; an exact-supervision design that removes the
sampling bottleneck entirely is the
subject of follow-up work.

This section places our results inside a live debate: physics-inspired GNN solvers for
combinatorial problems \citep{schuetz2022combinatorial} versus the critique that they
underperform trivial classical heuristics \citep{angelini2023modern}. The boundary curve
sides with the critics for from-scratch random CSPs and identifies the mechanism
(no depth-0 elimination signal plus miscalibrated sampled-solution value labels),
and the Sudoku-family results show where the same architecture \emph{does} work:
clue-rich completion, where amortized propagation does all the work. Out of this comes a
division of labor we develop in follow-up work: exact propagation and
verification from classical solving, learned components only for the decisions classical
solvers already delegate to heuristics \citep{bengio2021ml,selsam2019neurocore}, with
survey-propagation style message passing \citep{mezard2002analytic} as the natural
learned-decimation target at scale.

\section{Related work}
\label{sec:related}

\textbf{Do networks learn search?} Searchformer \citep{lehnert2024searchformer} and Stream
of Search \citep{gandhi2024sos} train transformers on serialized search traces and ask
whether search behavior is internalized; our question is the complement, whether a solver
\emph{built} around search uses it, and the answer here is that amortization displaces it.
Deep-thinking recurrent networks \citep{schwarzschild2021learn,bansal2022endtoend} are the
counterpoint on the recurrent side: trained for extrapolation, their iterations do carry
computation, which matches our boundary finding that iteration earns its keep only where
one-shot amortization \citep{gershman2014amortized} fails.

\textbf{Recurrent tiny reasoners.} HRM \citep{wang2025hrm} introduced the two-timescale
recurrent architecture; TRM \citep{jolicoeur2025trm} showed that one tiny recursive network
suffices; GRAM \citep{baek2026gram} made the recursion stochastic with variational guidance
and trajectory-ranking heads; LDT \citep{davis2026ldt} grounded the recursion in a sound
candidate-set lattice. \colt{} is a synthesis: LDT's envelope, GRAM's
learned-search-as-distribution perspective, and a CDCL-inspired control structure
(chronological, leaf-level; Limitations~iii).

\textbf{Neural CSP and Sudoku solvers.} NeuroSAT \citep{selsam2019neurosat} learned SAT
solving end-to-end via message passing but with no soundness guarantee; Recurrent Relational
Networks \citep{palm2018rrn} reached 96.6\% on the hardest published Sudoku split in 2018,
which is why none of our accuracy numbers are presented as capability news; SATNet
\citep{wang2019satnet} differentiates through an SDP relaxation of MAXSAT; RUN-CSP
\citep{toenshoff2021runcsp} trains message passing for max-CSP. None of these abstain or
guarantee soundness, the axis this paper studies; the anatomy question here is what the
sound, abstaining variant of this model
class computes. Modern CDCL solvers \citep{silva1999grasp,moskewicz2001chaff} are
sound and complete but hand-coded; \colt{} occupies the middle: learned propagation and
ordering, exact emission gate, classical search skeleton.

\textbf{Learned heuristics inside exact solvers.} Learning the variable-ordering or
branching heuristic of an otherwise exact solver is an established program: learned
branching for MIP \citep{khalil2016learning}, GCNN imitation of strong branching
\citep{gasse2019exact}, neural variable selection in SAT \citep{selsam2019neurocore}, GNN
initialization of CDCL branching \citep{wang2024neuroback,eriksson2026ranking} (helpful,
but regime-dependent and eventually overridden by the solver's dynamic heuristics), and
the survey of \citet{bengio2021ml}. Our negative result is complementary: in the
verify-or-abstain lattice regime the learned-search slot is accuracy-inert because the
propagator amortizes the whole solution; the learned-heuristic program applies where search
actually iterates, which our boundary section (\S\ref{sec:boundary}) and follow-up work
target.

\textbf{Graph-structured attention.} The relational bias follows Graphormer
\citep{ying2021graphormer}; applying it to the \emph{constraint} graph of a CSP makes
the checkpoint geometry-free. Policy/value-guided search echoes AlphaZero
\citep{silver2018alphazero}; the difference is that our policy and value targets
are obtained directly from the lattice supervision without self-play (Eq.~\ref{eq:alpha},
Eq.~\ref{eq:pstar}); they are exact when the represented solution set is complete and
sampled approximations otherwise (\S\ref{sec:boundary}). Neural algorithmic reasoning
\citep{velickovic2021nar} pursues the same end (size-invariant learned executors of
algorithmic computations) from the trajectory-supervision side.

\textbf{Test-time augmentation.} The union symmetry-frame ensemble of \S\ref{sec:amortized}
is test-time augmentation \citep{shanmugam2021tta} over an exact symmetry group; the
mechanism-specific ingredient is the \emph{union} aggregation, chosen because the failure
mode is one-sided (false eliminations), where mean-aggregation TTA underperforms it by
construction and in measurement (100\% vs.\ 97.8\%).

\section{Limitations}
\label{sec:limitations}

(i) The 6$\times$6 headline, the 9$\times$9 $\pm$augmentation contrast, the hard-slice
anatomy, and the $c{=}3.0$ policy grid are reported over three training seeds; the
remaining cells (the 2$\times$2 factorial, transfer probes, other boundary densities) are
single-seed, with their repetition in the revision protocol (E3). (ii) Two of our
three search components were \emph{accuracy-inert} at 6$\times$6 (the policy head
everywhere, DFS/nogoods beyond efficiency) because the propagator either solves a puzzle
nearly alone or first-pass-poisons it (\S\ref{sec:anatomy}); the components' accuracy value,
if any, must be demonstrated in regimes with long sound search chains, which our 6$\times$6
data does not provide. (iii) Our conflict-driven learning is structurally limited by the
propagator being a black box: leaf-level nogoods (complete, verifier-rejected grids) are the
strongest \emph{sound} nogoods such a system admits, because lifting a conflict to a general
clause requires an explanation of which inferences produced it, and a neural elimination
offers none. CDCL-style conflict analysis therefore becomes available once
propagation is exact and explainable, as in lazy clause generation
\citep{ohrimenko2009lcg}; this is a design argument, not deferred work. (iv) Generated
unique-solution Sudoku is not the
Sudoku-Extreme corpus; cross-paper number comparisons are out of scope.
Moreover, the 6$\times$6 tier admits no symmetry-disjoint split at all (the domain has on
the order of fifty essential grids; leakage audit, \S\ref{sec:protocol}), so 6$\times$6
generalization is within-symmetry-class by construction and class-level generalization
rests entirely on the 9$\times$9 splits, which the audit shows are clean. (v) The
GRAM baseline is budget-starved by 10--100$\times$ relative to its published recipe
(\S\ref{sec:protocol}). (vi) The
multi-solution ($K>1$) path of the $\alpha$-operator is now exercised on graph coloring
(\S\ref{sec:coloring}) but at modest instance size; Maze-style tasks with $K$ in the hundreds
remain untested. (vii) The global heads (conflict, policy) do not transfer zero-shot across
board sizes (AUC 0.50); only the relational propagator does, and size-invariant global
readouts are an open architectural item.

\section{Reproducibility}
All numbers are produced by committed commands; datasets are deterministic with pinned MD5s;
the protocol, budgets, and success criteria were frozen before training
(\texttt{BENCHMARKS.md}); baselines are our public reimplementations
\citep{komissarov2026gram_reimpl,komissarov2026ltd_reimpl} on byte-identical data. One
loop over two CLI flags reproduces the full ablation grid. This version's state
is pinned by the immutable Git tag \texttt{r14-2026-07-15} of the
public repository;
all tables are regenerable from the committed raw JSONs in \texttt{results/}
(\texttt{scripts/make\_tables.py}), and the train/test splits carry the leakage audit of
\S\ref{sec:protocol} as a committed artifact. \emph{Independent verifier audit:} a
clean-room verifier, implemented without importing the solver or evaluation verifier code,
independently checked 25{,}231 valid and invalid cases with zero disagreements
(\texttt{results/verifier\_audit.json}; scope and method in \texttt{AUDIT.md}). The 6$\times$6 numbers of
\S\ref{sec:results}--\ref{sec:amortized} all come from one checkpoint evaluated in one
environment (\texttt{scripts/run\_canonical\_cpu.sh}). The canonical checkpoint is the
frozen protocol's original seed, 42, retrained with the committed recipe; it was
designated by that protocol before any evaluation, not selected among retrains, and all
three seeds are reported alongside it (its standard-slice count moved from 178/180 in the
GPU era to 180/180 here, within the reported seed spread of 178--180); the prospectively specified E4
check verified the two first-pass probe code paths bitwise identical on it
(\texttt{results/reconcile\_anatomy\_h2.json}). Cross-environment replication: GPU-era
runs of the same recipe reproduce the standard-slice count within two puzzles (178 vs.\
180 of 180), the hard-slice count within one (138 vs.\ 137), and the poisoning dichotomy
exactly, with 3--4 margin-borderline puzzles shifting between environments (committed
artifacts \texttt{colt6\_*.json}, \texttt{anatomy6hard.json},
\texttt{h2\_colt6hard\_union.json}). The 9$\times$9 arm now spans four environments: the original GPU run (173/180), an
independent same-recipe retrain on a second GPU host (173/180), and three fresh seeds on a
third host (174/173/174 of 180), whose checkpoints re-evaluated bit-identically after a
mid-experiment pod replacement (a fourth host and a different torch version).

\section*{Acknowledgements}
This work began as a discussion during a break at the HAICON26 conference and was carried
out over the course of the conference itself. We thank Leonid Didukh for that discussion.

\section*{AI-Assisted Research Statement}
This work was carried out end to end by AI systems: Claude Opus 4.8 and Claude Fable 5
(Anthropic), operating as autonomous research agents under the direction of the human
author, who set the research goals, supplied and approved the compute resources, made the
scope and go/no-go decisions, and reviewed the outputs. The agents wrote all code, designed
and froze the experimental protocol, executed and monitored all training and
evaluation runs, performed the analyses, including the decision to pursue the failure-anatomy
line, and drafted this manuscript. The initial study was completed in roughly two working
days; the revision experiments reported in this version added multi-seed training, budget
sweeps, boundary runs, and independent verification across several CPU and GPU hosts, with
run logs and configurations committed in the repository. The human author takes
responsibility
for all published claims. Reliability is the main epistemic concern with agent-produced
research, so we state what was independently checked: the human
author reviewed the manuscript, the headline numbers against the raw JSONs, and the
go/no-go decisions at each phase, but did not review all code line by line; the exact
emission gate is covered by unit tests and was independently cross-checked against a
clean-room second verifier on 25{,}231 cases with zero disagreements; its correctness is
additionally implied by each emitted grid passing independent re-verification at
evaluation time; tables are regenerated from committed artifacts by script; and the one
cross-host retrains reproduced the headline counts within the reported seed spreads
(see Reproducibility). Residual risk concentrates where review was thinnest, the analysis
scripts, which is one reason the revision protocol re-runs the central probes in a single
environment. The frozen protocol, commit history, and run logs in
the repository form a complete audit trail of the agents' decisions, including the
frozen criteria that failed and are reported as failures.

\bibliographystyle{plainnat}
\bibliography{refs}

\appendix
\section{Claim audit: each causal claim and its instrument}
\label{app:audit}

This appendix lists each causal claim in the paper, the measurement that supports it, and
the committed artifact holding the raw numbers. The two non-empirical claims are
Observation~\ref{obs:licensing} (definitional) and Proposition~\ref{prop:symmetry} (proved).

\begin{center}
\footnotesize
\begin{tabular}{p{0.34\textwidth}p{0.33\textwidth}p{0.26\textwidth}}
\toprule
\textbf{claim} & \textbf{instrument} & \textbf{artifact (\texttt{results/})} \\
\midrule
\colt{} 98.9--100\% over seeds; +49.5--50.6\,pp vs.\ the fixed historical baseline & frozen-budget grid, 3 seeds, one environment & \texttt{colt6cpu\_*.json}, \texttt{ablate6\_F\_*.json} \\
within-codebase graph bias suffices for full-\colt{} performance & E8: six-arm single-component ablation, 3 seeds, one env. & \texttt{ablate6\_*.json} \\
search components accuracy-inert & flat 2$\times$3 grids on two slices + budget sweep (5/15/60 rounds) & \texttt{colt6cpu\_*.json}, \texttt{colt6cpu\_hard\_budget*.json} \\
backtracking + nogoods cut waste 1{,}497$\times$ & suppressed-completion counters, matched budgets & \texttt{colt6cpu\_hard\_*.json} \\
failure $\Leftrightarrow$ first-pass poisoning & perfect contingency on all seeds (43/43, 49/49, 38/38 vs.\ 0 clean failures) & \texttt{anatomy6hard\_cpu*.json} \\
probes bitwise-identical; GPU-era runs replicate the dichotomy & E4 bitwise check; earlier committed artifacts & \texttt{reconcile\_anatomy\_h2.json}, \texttt{anatomy6hard.json} \\
poisoning is confident, not borderline & $\theta_{\text{elim}}$ invariance sweep & \texttt{reconcile\_anatomy\_h2.json} \\
one-shot amortization & first-pass singleton / commit-precision rates & \texttt{h1\_*.json} \\
train-time cure (augmentation) & 9$\times$9 $\leq$0.6\% $\to$ 96.5 $\pm$ 0.3\% at fixed budget, 3 seeds & \texttt{colt9\_*\_s4*.json} \\
test-time cure; union is the active ingredient & $K$-frame sweep, union vs.\ mean ablation & \texttt{h2\_colt6cpu\_hard\_*.json} \\
structure $\times$ augmentation interaction & 2$\times$2 factorial at matched budget & \texttt{eval\_ldt6aug.json}, \texttt{colt6aug\_*.json} \\
local propagator transfers; global heads do not & zero-shot 9$\times$9 elimination P/R, AUC, policy margin & \texttt{transfer9\_multi.json} \\
boundary curve; denser supervision gives the largest gain among tested crossing interventions & density sweep + one-axis-at-a-time crossing arms & \texttt{phase\_transition*.json}, \texttt{boundary\_cross.json} \\
positional tables reach their ceiling at 10$\times$ steps & budget sweep 5k/15k/50k, 3 seeds at 50k & \texttt{armA\_*.json}, \texttt{budget\_sweep\_*.json} \\
poisoning is value-specific; wrapper alone inert & frame-type factorial with identity predicted-null & \texttt{h2\_factorial\_cpu.json} \\
learned policy $\approx$ MRV where branching matters & full policy grid at $c{=}3.0$, 3 seeds & \texttt{policy\_grid\_c30\_*.json} \\
transfer is mostly relearned propagation & rule-based floor + beyond-propagation precision/recall & \texttt{transfer9\_floor.json} \\
verifier independently validated & clean-room second verifier, 25{,}231 cross-checks & \texttt{verifier\_audit.json} \\
classical solver closes all test sets without search & per-instance timing and decision counts & \texttt{classical\_reference.json} \\
no train--test leakage at 9$\times$9; 6$\times$6 is within-class by domain & symmetry-exact orbit audit, full group enumerated & \texttt{leakage\_audit.json} \\
cross-host reproducibility & multi-host, multi-seed replication within reported spreads & \texttt{colt9\_*\_s4*.json}, GPU-era artifacts \\
\bottomrule
\end{tabular}
\end{center}

\clearpage
\section{Hyperparameters and budgets}
\label{app:hyper}

All Sudoku-family models in \S\ref{sec:results}--\ref{sec:phase4} share one recipe
(\texttt{colt6.yaml} and \texttt{colt\_multi.yaml} in \texttt{configs/}); the boundary study
(\S\ref{sec:boundary}) states its own deviations inline.

\begin{center}
\small
\begin{tabular}{llll}
\toprule
\textbf{model} & & \textbf{training} & \\
\midrule
$d_{\text{model}}$ / heads / layers & 64 / 4 / 4 & optimizer & AdamW, lr $3{\cdot}10^{-3}$ \\
recurrent iterations per pass $L$ & 10 & weight decay / grad clip & 0.1 / 1.0 \\
FF multiplier / $V_{\max}$ padding & 4 / 9 & betas & (0.9, 0.95) \\
9$\times$9 (Phase 4) scale-up & $d{=}128$, $L{=}16$ & batch / steps (6$\times$6) & 64 / 5{,}000 \\
 & & batch / steps (9$\times$9) & 256 / 20{,}000 \\
\midrule
\textbf{lattice / solver} & & \textbf{loss} & \\
\midrule
$\theta_{\text{elim}}$ & 0.1 & elimination BCE $w_{+}/w_{-}$ & 4.0 / 0.5 \\
$\theta_{\text{cls}}$ (train / eval) & 0.9 / 0.6 & $\lambda_{\text{cls}}$ / $\lambda_{\text{ce}}$ / $\lambda_{\text{policy}}$ & 0.1 / 0.2 / 0.25 \\
branch temperature $\tau$ & 1.5 & curriculum reveal max / anneal & 0.9 / 0.7 \\
pool age limit $\tau_{\text{age}}$ & 60 (100 at 9$\times$9) & warmup fraction & 0.03 \\
\bottomrule
\end{tabular}
\end{center}

A ``round'' at inference is one forward pass of all chains followed by threshold
elimination, conflict handling, and (if no cell changed) one branching commitment per
chain; the frozen budgets are 32 chains $\times$ 60 rounds (4$\times$4, 6$\times$6) and 64
chains $\times$ 200 rounds (9$\times$9). Batch size at 9$\times$9 is 256 where the LDT
paper's recipe uses 512: full-unroll BPTT at batch 512 exceeds a 24\,GB card, and this is
the reproduction's documented deviation alongside the omitted augmentation in the baseline
row.

\end{document}